\documentclass[letterpaper, 10 pt, journal, twoside]{IEEEtran}
\usepackage{times}
\usepackage{multicol}
\usepackage[bookmarks=true]{hyperref}
\usepackage{amsfonts}
\usepackage{amsmath}
\usepackage{amstext}
\usepackage{amsthm}
\usepackage{epsfig} 
\usepackage{xcolor}
\usepackage{mathtools}
\usepackage[font=footnotesize]{caption}
\usepackage{subcaption}
\usepackage{xurl}
\usepackage{multirow}
\let\labelindent\relax

\usepackage{enumitem,kantlipsum}
\usepackage{float}

\usepackage{algorithm} 
\usepackage{algpseudocode} 

\usepackage{booktabs}
\usepackage{multirow}
\usepackage{graphicx}

\IEEEoverridecommandlockouts                              

\ifCLASSINFOpdf
\else
\fi

\begin{document}

\newcommand{\sai}[1]{\textcolor{blue}{Sai: #1}}

\newcommand{\sharath}[1]{\textcolor{red}{Sharath: #1}}

\newcommand{\sandeep}[1]{\textcolor{green}{#1}}

\newcommand{\philip}[1]{\textcolor{red}{#1}}

\newcommand{\todo}[1]{\textcolor{red}{TODO: #1}}

\newcommand{\coolname}{\emph{Delayed Communication} \textrm{MDP}}

\newcommand{\coolnameabbr}{\textrm{DC$-$MDP}}

\newcommand{\rebuttal}[1]{\textcolor{blue}{#1}}


\newcommand{\mdptuple}{\langle \mathrm{S}, \mathrm{Init}, \mathrm{Act}, \mathbb{A}, \mathbb{P} \rangle}

\newcommand{\bad}{\mathrm{{S}_{unsafe}}}

\newcommand{\piprpsi}{\mathrm{V_{\mathcal{M}, \theta}^\pi}}

\newcommand{\piprphi}{\mathrm{V_{\mathcal{M}, \varphi}^\pi}}

\newcommand{\piprphieps}{\mathrm{V_{\mathcal{M}, \varphi}^{\pi_{\epsilon}}}}

\newcommand{\minprphi}{\mathrm{V_{\mathcal{M}, \varphi}^{min}}}

\newcommand{\maxprphi}{\mathrm{V_{\mathcal{M}, \varphi}^{max}}}

\newcommand{\optsafepolicy}{\pi_{\mathcal{M}, \varphi}^{\mathrm{safe}}}

\newcommand{\mininitsafety}{\mathbb{E}_{s \sim \mathrm{Init}} \left [ \mathrm{V_{\mathcal{M}, \varphi}^{min}}(s) \right ]}

\newcommand{\piinitsafety}{\mathbb{E}_{s \sim \mathrm{Init}} 
\left [ 
\piprphi(s)
\right ]}

\newcommand{\maxinitsafety}{\mathbb{E}_{s \sim \mathrm{Init}} \left [ \mathrm{V_{\mathcal{M}, \varphi}^{max}}(s) \right ]}

\newcommand{\piepsinitsafety}{\mathbb{E}_{s \sim \mathrm{Init}} \left [ \mathrm{V_{\mathcal{M}, \varphi}^{\pi_{\epsilon}}}(s) \right ]}

\newcommand{\minqphi}{{\mathrm{Q_{\mathcal{M}, \varphi}^{min}}}(s,a)}

\newcommand{\maxqphi}{{\mathrm{Q_{\mathcal{M}, \varphi}^{max}}}(s,a)}


\newcommand{\picloud}{\pi_{\mathrm{cloud}}}

\newcommand{\basicmdp}{\mathcal{M}_b}

\newcommand{\epshield}{\epsilon $-$\mathrm{shield}}

\newcommand{\epshieldsym}{\mathrm{C}_\epsilon}

\newcommand{\epstarshield}{\epsilon^\ast $-$\mathrm{shield}}

\newcommand{\epstarshieldsym}{\mathrm{C}_{\epsilon}^\ast}

\newcommand{\epsmdptuple}{\langle \mathrm{S}, \mathrm{Init}, \mathrm{Act}, \mathrm{C}_\epsilon, \mathbb{P} \rangle}

\newcommand{\epsbarmdptuple}{\langle \mathrm{S}, \mathrm{Init}, \mathrm{Act}, \mathrm{C}_{\bar{\epsilon}}, \mathbb{P} \rangle}

\newcommand{\minprphieps}{\mathrm{V_{\mathcal{M}_\epsilon, \varphi}^{\mathrm{min}}}}

\newcommand{\maxprphieps}{\mathrm{V_{\mathcal{M}_\epsilon, \varphi}^{\mathrm{max}}}}

\newcommand{\piprphiepsbar}{\mathrm{V_{\mathcal{M}_{\bar{\epsilon}}, \varphi}^\pi}}

\newcommand{\minprphiepsbar}{\mathrm{V_{\mathcal{M}_{\bar{\epsilon}}, \varphi}^{\mathrm{min}}}}

\newcommand{\actinitsafety}{\mathbb{E}_{s \sim \mathrm{Init}} \left [ \mathrm{V_{\mathcal{M}, \varphi}^\pi}(s) \right ]}

\newcommand{\actinitsafetyeps}{\mathbb{E}_{s \sim \mathrm{Init}} \left [ \mathrm{V_{\mathcal{M}, \varphi}^{\pi_\epsilon}}(s) \right ]}

\newcommand{\mininitsafetyeps}{\mathbb{E}_{s \sim \mathrm{Init}} \left [ \mathrm{V_{\mathcal{M}_\epsilon, \varphi}^{\mathrm{min}}}(s) \right ]}

\newcommand{\R}{\mathbb{R}}

\newcommand{\stateset}{\mathrm{S}}

\newcommand{\pr}{\mathrm{Pr}}

\newcommand{\initset}{\mathrm{Init}}

\newcommand{\maxcommdelay}{\mathrm{\tau_{max}}}

\newcommand{\ptau}{\mathbb{P_\tau}}

\newcommand{\delaytransition}{\ptau: \Omega \times \Omega \rightarrow [0,1]}

\newcommand{\mdptransition}{\mathbb{P}:\mathrm{S} \times \mathrm{Act} \times \mathrm{S} \rightarrow [0,1]}

\newcommand{\labelfn}{\mathrm{L}:\mathrm{S} \rightarrow 2^{\mathrm{AP}}}

\newcommand{\mdp}{\mathcal{M}}

\newcommand{\latencymodel}{\mathbb{P}_\tau}

\newcommand{\NetworkedMdpTuple}{\langle \mathrm{X}_d, \mathrm{Init}_d, \mathrm{Act}, \mathbb{A}, \mathbb{P}_d \rangle}

\title{Safe Networked Robotics with Probabilistic Verification}
\author{Sai Shankar Narasimhan$^{\ast}$, Sharachchandra Bhat$^{\ast}$, and Sandeep P. Chinchali
\thanks{Sai Shankar Narasimhan, Sharachchandra Bhat, and Sandeep P. Chinchali are with the Department of Electrical and Computer Engineering, The University of Texas at Austin, USA.  {\tt\footnotesize \{nsaishankar, sharachchandra, sandeepc\} @utexas.edu}. $^{\ast} $ denotes equal contribution.}%
}


\newtheorem{definition}{\textbf{Definition}}
\newtheorem{theorem}{\textbf{Theorem}}
\newtheorem{proposition}{\textbf{Proposition}}
\newtheorem{remark}{\textbf{Remark}}

\maketitle

\begin{abstract}
 Autonomous robots must utilize rich sensory data to make safe control decisions. To process this data, compute-constrained robots often require assistance from remote computation, or the cloud, that runs compute-intensive deep neural network perception or control models. However, this assistance comes at the cost of a time delay due to network latency, resulting in past observations being used in the cloud to compute the control commands for the present robot state. Such communication delays could potentially lead to the violation of essential safety properties, such as collision avoidance. This paper develops methods to ensure the safety of robots operated over communication networks with \emph{stochastic} latency. To do so, we use tools from formal verification to construct a shield, i.e., a run-time monitor, that provides a list of safe actions for any delayed sensory observation, given the expected and maximum network latency. Our shield is minimally intrusive and enables networked robots to satisfy key safety constraints, expressed as temporal logic specifications, with desired probability. We demonstrate our approach on a real F1/10th autonomous vehicle that navigates in indoor environments and transmits rich LiDAR sensory data over congested WiFi links.

\end{abstract}


\IEEEpeerreviewmaketitle

\section{Introduction}
\label{sec:intro}

\IEEEPARstart{T}{oday}, an increasing number of robotic applications require remote assistance, ranging from remote manipulation for surgery 
\cite{el2020review} to emergency take-over of autonomous vehicles \cite{el2022ar}. Teleoperation is even used to control food delivery robots from command centers hundreds of miles away \cite{cocofood}. In these scenarios, network latency is a key concern for safe robot operation since actuation based on delayed state information can lead to unsafe behavior.

\begin{figure}[!h]
    \centering
    \includegraphics[width=1.0\columnwidth]{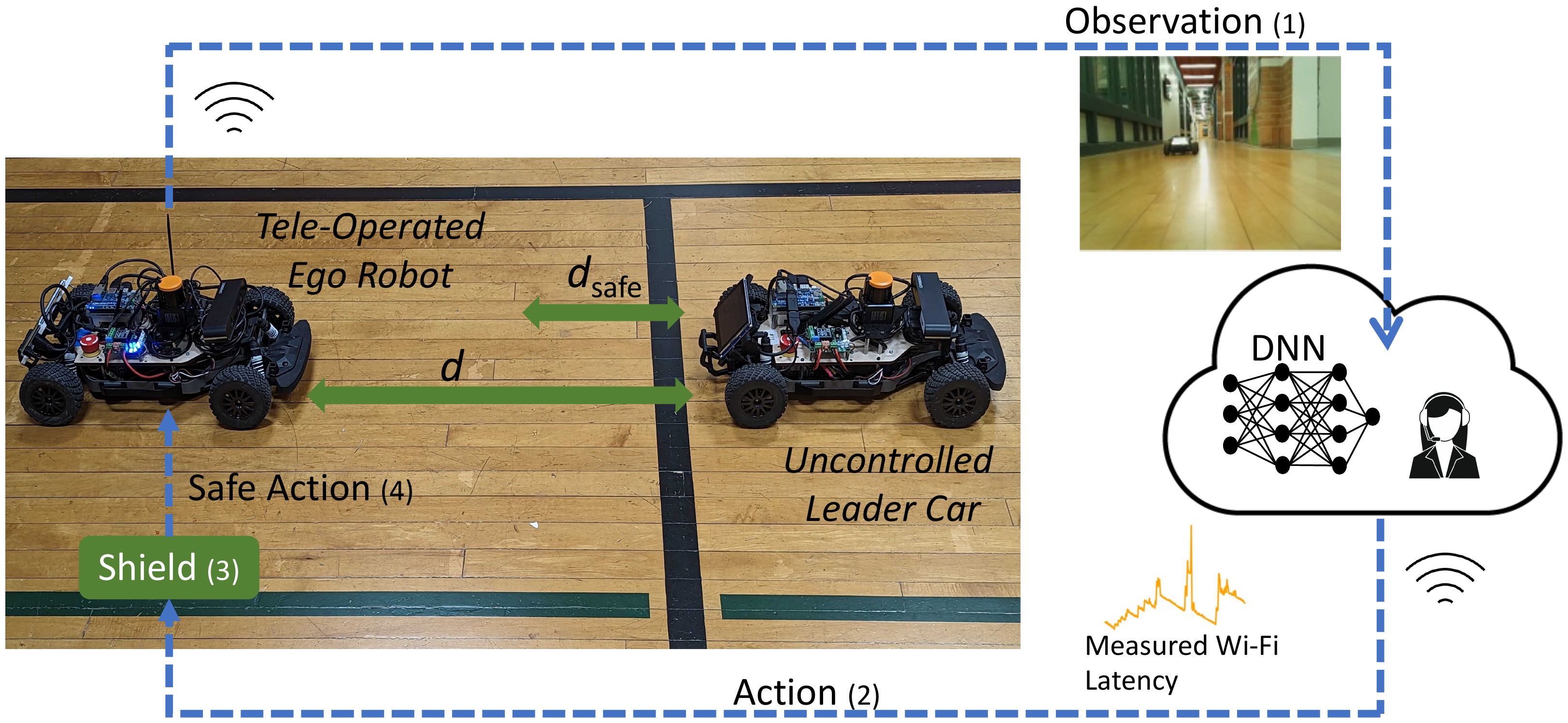}
    

    \caption{\textbf{Safe Networked Control for Robotics:} A resource-constrained robot transfers sensor observations (RGB-D images or LiDAR point clouds) through a wireless network with stochastic latency. At the receiving end, a control module or a human teleoperator processes the observation to generate the corresponding action. The action is filtered by the shield, which enforces a particular safety specification that the robot has to maintain. The filtered, ``safe'' action is then executed by the robot.}
    \label{fig:introduction_figure}
\end{figure}

Despite the rise of robots operating over communication networks, we lack formal guarantees for their safe operation. Today’s approaches for robotic safety range from reachability analysis \cite{bansal2017hamilton, Fisac2019BridgingHS, Cheng2019EndtoEndSR, choi2020reinforcement} to shielding that restricts unsafe actions based on a formal safety specification \cite{alshiekh2018safe, carr2023safe, jansen2020safe, konighofer2020shield}. However, there is little to no research that provides such rigorous safety analysis for networked robotics. This paper asks: \emph{How do we ensure safe networked control over wireless networks with stochastic communication delays?}

Communication delay is the cumulative time taken to send an observation to the cloud and receive an action back at the robot. We develop the intuition that if the interaction between a remotely controlled robot and its environment can be modeled as a Markov Decision Process (MDP), the communication delay is analogous to sensing or actuation delays. Previous works on Networked Control Systems (NCS) have addressed MDPs with delays \cite{adlakha2011networked,katsikopoulos2003markov,derman2020acting}, but often make restrictive assumptions about delay transitions. In our paper, we propose \coolname{}, a novel approach to model MDPs with delays that aligns naturally with the transmission of observations and control commands when operating a robot via wireless networks in practice.

Fig. \ref{fig:introduction_figure} shows our approach, tested on an F1/10th car \cite{o2019f1} controlled over a wireless link. Our approach is extremely general - we can either have a human teleoperator or an automatic controller running in the cloud, including Deep Neural Network (DNN) perception models or deep reinforcement learning (RL) based policies. First, sensor observations are transferred via wireless links (step 1) and processed to compute the corresponding control command (step 2). The control command is transmitted back to the robot and filtered by the shield. The shield is a run-time monitor, constructed offline, that disallows actions that violate a safety property. The shield has access to the delay corresponding to the received control command as it runs on the robot. Finally, the shielded action is executed to ensure safe behavior amidst stochastic network latency (steps 3-4). 
We design the shield using tools from formal verification \cite{baier2008principles}, given knowledge of the network latency and a model of the environment transitions.

Shields, as implemented in \cite{alshiekh2018safe}, provide an absolute measure of safety. For networked control with stochastic latency, shielding results in perfectly safe operation at the cost of task efficiency. In this paper, we propose a shield synthesis approach that, when combined with the cloud controller, allows the networked control system to meet safety requirements with a desired probability. Our experimental findings indicate that a slight reduction in the desired safety probability leads to a significant increase in task efficiency. In this paper,
\begin{enumerate}
    \item We present the \coolname, a novel approach that accurately models the interaction between a remotely controlled robot and the environment, in the presence of stochastic network latency.
    \item  We propose an algorithm to synthesize a shield that, when executed with the cloud controller, guarantees the desired probability of satisfying a safety property.  
    \item We demonstrate our approach in simulation as well as on an F1/10th autonomous vehicle that must closely (and safely) follow an unpredictable leader in indoor environments over congested wireless networks (Fig. \ref{fig:introduction_figure}). 
\end{enumerate}

\section{Related Work}
\label{sec:related_work}
We now survey how our work relates to cloud robotics, networked control systems, shielding, and formal methods. 

\emph{Cloud Robotics:} Cloud robotics \cite{kehoe2015survey,kuffner2010cloud} studies how resource-constrained robots can offload inference \cite{tanwani2019fog,chinchali2021network}, mapping, and control to remote servers \cite{mohanarajah2015rapyuta}. Recent work (\cite{tian2022mitigating}) circumvents network latency for teleoperation by predicting the intent of a teleoperator remotely and synthesizing trajectories locally on the robot for handwriting imitation. This approach does not scale well for resource-constrained robots for more complicated tasks like autonomous driving as it involves running DNNs for intent prediction.

\emph{Delayed MDPs:} Numerous prior works have addressed sensing and actuation delays in MDPs \cite{adlakha2011networked,katsikopoulos2003markov,derman2020acting,lancewicki2022learning}, by making restrictive assumptions about the delay transitions. The delay is constant in \cite{adlakha2011networked} and \cite{derman2020acting}, while it can only increase or remain constant between consecutive time steps in \cite{katsikopoulos2003markov}, halting the decision-making when the delay reaches the maximum limit. In contrast, we make no such assumptions. Additionally, prior works have focused on the delay in feedback \cite{lancewicki2022learning} or cost collection (\cite{katsikopoulos2003markov}, \cite{derman2020acting}) in the RL setting. Our objective is to \emph{formally verify the safety of NCS}.

\emph{Shielding and Safe Reinforcement Learning:} Our work builds upon safe RL techniques developed for discrete-time systems. The shielding approach (\cite{alshiekh2018safe, konighofer2020shield}) involves synthesizing a run-time monitor that overwrites the agent's action if it violates the desired safety specification, aiding in safe exploration \cite{carr2023safe}. Recent work relaxes assumptions on the knowledge of the environment and makes the shielding approach more practical \cite{pranger2021, konighofer2022online}. As the execution of perfectly safe policies restricts exploration in RL, probabilistic shields were introduced in \cite{jansen2020safe} and \cite{bouton2019reinforcement} to trade off safety for exploration. Building on this, recent work implements probabilistic shields through probabilistic logic programs \cite{yang2023safe}. The definitions of the probabilistic shield in \cite{jansen2020safe} and \cite{bouton2019reinforcement} are similar to ours, but these works do not provide theoretical guarantees for safety, which we do. Another probabilistic shielding approach \cite{aksaray2021probabilistically} focuses specifically on synthesizing shields that satisfy \emph{bounded} specifications. While the above-mentioned works deal with shielding for safe RL, our work focuses on developing a novel shielding approach for NCS.

 The previous approaches deal with finite state models obtained from an abstraction of the continuous state space. \cite{li2020robust} proposes an alternative approach using Robust Model Predictive Shielding. Further, Hamilton-Jacobi reachability analysis \cite{bansal2017hamilton, Fisac2019BridgingHS} and Control Barrier Function methods \cite{Cheng2019EndtoEndSR, choi2020reinforcement} formulate the safe control problem for the continuous system. These methods cannot express rich safety properties, such as ``maintain a minimum distance between two vehicles when the delay is above a threshold and visit the landmark before reaching the goal'', which is possible in our approach.

\section{Background}
\label{sec:background}

A \emph{Markov Decision Process} (MDP) is a tuple $\mdptuple$, where $\stateset$ is a finite state set, $\initset$ is a probability distribution over $\stateset$ representing the initial state distribution, and $\mathrm{Act}$ is a finite set of actions. The transition probability function $\mdptransition$ is a conditional probability distribution and hence satisfies $\sum_{{s}'\in \mathrm{S}}\mathbb{P}({s'}\mid s,a) = 1$ for every state-action pair $(s,a) \in \stateset \times \mathbb{A}(s)$, where $\mathbb{A}(s)=\left \{a \in \mathrm{Act} \ | \ \exists \ s' \in \stateset \ s.t. \  \mathbb{P}({s}' \mid s,a) \neq 0 \right \}$ is the set of available actions for the state $s$. A policy $\pi$ is defined as a mapping from states to actions, $\pi: \stateset \rightarrow \mathrm{Act}$.

We will now introduce safety properties and our approach using the hardware setup in Fig. \ref{fig:introduction_figure}, where a resource-constrained mobile robot must safely follow an unpredictable leader while being controlled remotely over a wireless link with stochastic network delays. Henceforth, we will use the term \emph{agent} to indicate any controlled entity like the robot and \emph{environment} for uncontrolled entities (like the leader car). The agent and the environment together form the \emph{system}. 

To capture our desired notion of safety for the system, we first define $\bad$ to be the set of all \emph{unsafe states}. For example, in our hardware setup, an unsafe state is one where the distance between the two cars, $d$, is less than the safety threshold $d_{\mathrm{safe}}$ (Fig. \ref{fig:introduction_figure}). Then, we define the system to be safe if it never reaches any state in $\bad$. This can be encapsulated by the Linear Temporal Logic (LTL) \cite{baier2008principles} safety specification $\square \neg \bad$, which translates to ``always ($\square$) never ($\neg$) be in an unsafe state". Note that our notion of safety is now equivalent to determining the probability with which the system \emph{satisfies the safety specification} $\varphi = \square \neg \bad$, which can be done efficiently. We use $\piprphi(s)$ to represent the probability of satisfying $\varphi$, while executing the policy $\pi$ starting from the state $s \in \stateset$. The probability with which the system satisfies $\varphi$ is then given by the expectation of $\piprphi(s)$ over $\initset$. To compute $\piprphi$, we note that the safety specification $\varphi = \square \neg \bad$ can be cast into a reachability specification $ \theta = \diamond \bad$, which refers to ``eventually ($\diamond$) reach any unsafe state''. Now, the probability of satisfying this reachability specification, $\piprpsi(s)$, is the unique solution to the following system of equations \cite{baier2008principles}:

\begin{equation}
\begin{gathered}
    \textrm{if } s \in \bad \Rightarrow \piprpsi(s) = 1; \textrm{if }s \not \models \theta \Rightarrow \piprpsi(s) = 0,\\
    \textrm{else } \piprpsi(s) = \mathbb{E}_{{s}' \sim \mathbb{P}({s}' \mid s,\pi(s))} \left [ \piprpsi({s}') \right ].
\end{gathered}
\label{eq:reachability_analysis}
\end{equation}
This can be solved using value iteration. Then, the safety probabilities can be computed using the relation $\piprphi(s) = 1 - \piprpsi(s) \ \forall s \in \stateset$. We denote the minimum and maximum safety probabilities across any policy as $\minprphi(s)$ and $\maxprphi(s)$, respectively. We refer the readers to \cite{baier2008principles} for details on how they can be computed.  We also denote the minimum and maximum safety probabilities for a \emph{state-action pair} $(s,a) \in \stateset \times \mathbb{A}(s)$ by $\minqphi$ and $\maxqphi$ respectively. For example, $\maxqphi$ is computed as
\begin{equation}
\maxqphi = \mathbb{E}_{{s}' \sim \mathbb{P}({s}' \mid s,a)} \left [ \maxprphi({s}') \right ].
 \label{eq:qmax_computation}
\end{equation}
The above discussion can be extended to reach-avoid specification, $\neg \ \bad \ \cup \mathrm{goal}$, which translates to ``never ($\neg$) be in an unsafe state until ($\cup$) the goal state ($\mathrm{goal}$) is reached". We note that the reach-avoid specifications can also be cast into a reachability specification, and refer the readers to \cite{baier2008principles} for further details. We denote the policy corresponding to the maximum safety probability, $\maxprphi(s)$, as the optimally safe policy $\optsafepolicy$, defined as $\optsafepolicy(s) = \operatorname{argmax}_a \maxqphi$. These quantities are necessary to define our \emph{shield} that can ensure a desired safety probability $\delta$ for the networked control system. For an MDP $\mathcal{M} = \mdptuple$, a shield (\cite{alshiekh2018safe, jansen2020safe}) is a function, $\mathrm{C}: \stateset \rightarrow 2^{\mathrm{Act}}$, that maps every state $s \in \mathrm{S}$ to a subset of $\mathbb{A}(s)$. During runtime, the shield overwrites the policy only if $\pi(s) \not \in \mathrm{C}(s)$.

\section{Problem Formulation}
\label{sec:prob_formulation}
In this section, we formally define our safe networked control problem. We make the following three key assumptions:
\begin{itemize}[leftmargin=*]
    \item The agent-environment interaction is available as an MDP $\basicmdp = \mdptuple$, where the state and action sets are discrete and finite. For the continuous case, we obtain finite sets by abstracting the continuous state and action spaces. We term this as the \emph{Basic} MDP. In our hardware setup, the state set $\mathrm{S}$ consists of bins of possible distances between the cars, the action set $\mathrm{Act}$ consists of bins of allowed ego-robot velocities and the transition probability function $\mathbb{P}$ captures the leader's unpredictability modeled using an assumed range of velocities. This is a standard assumption since the offline computation of safe control policies typically require knowledge of the agent-environment interaction \cite{bansal2017hamilton, jansen2020safe, bouton2019reinforcement}.
    \item We assume a sufficient understanding of the stochasticity in communication delay, which we model as a transition probability function $\latencymodel$ with an upper bound $\maxcommdelay$ on the delay. Later, in Sec. \ref{sec:experiments}, we show how to obtain $\latencymodel$ from the collected time-series datasets of communication delays. 
    \item Finally, we assume the cloud controller $\picloud$ is available as a mapping from the state set $\mathrm{S}$ to the action set $\mathrm{Act}$ for the \emph{Basic} MDP. For the discrete case, this mapping is trivial as it is $\picloud$ itself. For the continuous case, the mapping can be easily obtained even for complex DNN controllers \cite{liu2021algorithms}. Later, we show how to relax this assumption for cases like human-teleoperation. Note that the cloud controller $\picloud$ is unaware of the communication delay.
\end{itemize}

We now explain the practical effects of delays on NCS. Consider an agent sending timestamped observations to the cloud. The cloud processes these observations to extract the system state information, generates a corresponding action, and appends the same timestamp to it before sending it back to the agent. We define communication delay as the time difference between the current time and the timestamp of the received action. Formally, at time $t$, the communication delay is $\tau_t$ if the received action, $a_t$, corresponds to the delayed state $s_{t-\tau_t}$. We refer to $s_{t-\tau_t}$ as the latest available system state at time $t$. Between two consecutive time steps $t$ and $t+1$, only one of the following three events can occur.
\begin{itemize}[leftmargin=*]
    \item \emph{Case 1:} The agent receives no action from the cloud. This implies that the latest available system state at $t+1$ is still $s_{t-\tau_t}$ and the delay $\tau_{t+1} = \tau_t + 1.$ 
    \item \emph{Case 2:} The agent receives an action with a timestamp equal to the current time, implying no delay, i.e., $\tau_{t+1} = 0.$
    \item \emph{Case 3:} The agent receives an action with an older timestamp, implying $\tau_{t+1} > 0$ and $\tau_{t+1} \leq \tau_t$. 
\end{itemize}

To model these events, we represent the delay transitions as a conditional probability distribution $\ptau(\tau_{t+1} \mid \tau_t)$, $\delaytransition$
where $\Omega = \left \{0,1,.., \maxcommdelay  \right \}$ is the set of integer delay values. As the delay cannot increase by more than 1 (\emph{Case 1}), we have  $\ptau(\tau_{t+1} \mid \tau_t)=0$ if $\tau_{t+1} > \tau_t + 1$.

\emph{\textbf{Problem: }}We are given the \emph{Basic} MDP $\basicmdp$, the delay transition probability function $\ptau$ with an upper bound on delay $\maxcommdelay$ and the cloud controller $\picloud$. Our aim is to ensure safe networked control such that the system satisfies the safety specification $\square \neg \bad$ with probability $\delta$, where $\bad$ denotes the unsafe states.
 
\section{Approach}
\label{sec:approach}
Our approach to safe networked control is based on shield construction. 
The shield construction for safe networked control requires an MDP that is cognizant of the delay. However, the \emph{Basic} MDP $\basicmdp$ does not account for any delay. Therefore, from $\basicmdp$, we first create a \coolname{} (\coolnameabbr{}) that accounts for the stochastic communication delay. This is outlined in Sec. \ref{sec:dcmdp}. Consequently, in Sec. \ref{sec:guaranteed_shield}, we describe our approach for shield construction for any MDP and any safety specification.

\subsection{Delayed Communication Markov Decision Processes}
\label{sec:dcmdp}
To design the \coolnameabbr{} from the \emph{Basic} MDP $\basicmdp$, we first note that in the presence of delay, the state transition model can no longer rely only on the current state $s_t$ and the current executed action $a_t$ to determine the next state $s_{t+1}$. This is because $s_t$ is not known when the delay is not zero. For a system with delay $\tau_t$ at time $t$, the maximum information available about the system at $t$ is the latest observed system state $s_{t-\tau_t}$ and the action buffer, i.e., sequence of actions executed from $t-\tau_t$ to $t-1$, $a_{t-\tau_t}, \dots ,a_{t-1}$. Therefore, determining whether an action $a_t$ is safe with respect to the property $\square \neg \bad$ should intuitively rely on $s_{t-\tau_t}$ and $a_{t-\tau_t}, \dots, a_{t-1}$. Hence, we incorporate this maximum information available about the system into the state of the \coolnameabbr. Note that the action buffer's length is the delay $\tau_t$, which can vary. So, we introduce $\maxcommdelay - \tau_t$ number of place-holder actions, $\phi$, to ensure the action buffer's length is always $\maxcommdelay$. Now, we define the state at time $t, x_t$, as $(s_{t-\tau_t}, (a_{t-\tau_t}, \dots, a_{t-1}, \phi, \dots,\phi), \tau_t)$ and the state space for the \coolnameabbr{} as $\mathrm{X}_d \in \stateset \times (\mathrm{Act} \cup \{ \phi \})^{\maxcommdelay} \times \Omega$, where $\Omega = \left \{0,1, \dots, \maxcommdelay  \right \}$ is the set of all possible delays. 
Without loss of generality, the initial delay, $\tau_0$ is $0$, i.e., the latest available system state at the beginning of any task execution is the initial system state. Let $s_{\mathrm{Init}}$ be the set of initial states for $\basicmdp$. Then we define the initial state probability distribution of the \coolnameabbr{}, $\mathrm{Init}_d$, to only have non-zero probabilities for the states in the list $s_{\mathrm{Init}} \times \left\{(\phi, \phi, \dots,\phi)\right\} \times \left \{0  \right \}$.

For the state $x_t = (s_{t-\tau_t}, (a_{t-\tau_t}, \dots, a_{t-1}, \phi, \dots, \phi), \tau_t)$ and the action $a_t$, we now relate the three possible events described in Sec. \ref{sec:prob_formulation} to the state transitions in $\mathrm{X}_d$. 

\begin{itemize}[leftmargin=*]
    \item \emph{Case 1:} $\tau_{t+1} = \tau_t+1$. The latest available system state remains the same. Thus $x_{t+1}$ is $(s_{t-\tau_{t}}, (a_{t-\tau_{t}}, \dots,a_{t}, \phi, \dots,\phi), \tau_{t}+1)$. The occurrence of this event is governed only by the delay transition, hence the probability of this event is $\ptau(\tau_t+1 \mid \tau_t)$. 
    \item \emph{Case 2:} $\tau_{t+1} = 0$. The latest available system state is the current system state $s_{t+1}$. Thus $x_{t+1}$ is $(s_{t+1}, (\phi, \dots, \phi), 0)$. This event depends on the delay transition with probability $\ptau(0 \mid \tau_t)$ and the system transition from $s_{t-\tau_t}$ to $s_{t+1}$ by executing $\tau_t+1$ actions $a_{t-\tau_t}, \dots, a_t$, with probability $\mathbb{P}(s_{t+1} \mid s_{t-\tau_t}, a_{t-\tau_t}, \dots, a_t)$.
    \item \emph{Case 3:} $0 < \tau_{t+1} \leq \tau_{t}$. The latest available system state is delayed by $\tau_{t+1}$. Thus $x_{t+1}$ is $(s_{t+1-\tau_{t+1}}, (a_{t+1-\tau_{t+1}}, \dots, a_t, \phi, \dots, \phi), \tau_{t+1})$. Similar to \emph{Case 2}, the occurrence of this event is governed by the delay transition with probability $\ptau(\tau_{t+1} \mid \tau_t)$ and the system transition with probability $\mathbb{P}(s_{t+1-\tau_{t+1}} \mid s_{t-\tau_t}, a_{t-\tau_t}, \dots, a_{t-\tau_{t+1}})$. 
\end{itemize}

\begin{equation}
\begin{gathered}
\hspace{-5.5cm}\mathbb{P}_d(x_{t+1} \mid x_t, a_t) = \\
\left\{
\begin{matrix}

\hspace{-5.8cm}\ptau(\tau_{t+1} \mid \tau_t),\\ 

\hspace{-5.3cm} \textrm{if}\;  \tau_{t+1}=\tau_t+1\\ 

\vspace{0.3cm}

\hspace{-0.1cm}x_{t+1}=(s_{t-\tau_{t}}, (a_{t-\tau_{t}}, \dots,a_{t}, \phi, \dots,\phi), \tau_{t}+1)\\ 

\hspace{-1.9cm}\ptau(\tau_{t+1} \mid \tau_t)
\hspace{-0.5cm}
\underbrace{\sum_{s_{t-\tau_t+1} \in \stateset}
y_{t-\tau_t}
\dots \sum_{s_{t} \in \stateset} y_{t-1} \ y_t
}_{\tau_t + 1\; \textrm{terms}}
,\\ 
\vspace{0.3cm} 
\hspace{-1.8cm} \textrm{if} \;\tau_{t+1} =0, x_{t+1}=(s_{t+1},(\phi, \dots,\phi), 0) \\ 

\vspace{0.0cm} 

\hspace{0.1cm}\ptau(\tau_{t+1} \mid \tau_t)
\hspace{-0.5cm}
\underbrace
{\sum_{s_{t-\tau_{t}+1} \in \stateset} y_{t-\tau_t} \dots
\hspace{-0.25cm}
\sum_{s_{t-\tau_{t+1}} \in \stateset}
y_{t-\tau_{t+1}-1} \ y_{t-\tau_{t+1}}
}_{\tau_t - \tau_{t+1} + 1\; \textrm{terms}}
,\\  
 
\hspace{-5.2cm} \textrm{if} \; 0 < \tau_{t+1} \leq \tau_{t}\\ 

\vspace{0.3cm}
\hspace{0.3cm}x_{t+1}=(s_{t+1-\tau_{t+1}}, (a_{t+1-\tau_{t+1}}, \dots,a_{t}, \phi,\dots,\phi), \tau_{t+1}) \\

\hspace{-5.5cm} 0 \;\;\;\;\;\; \textrm{otherwise},
\end{matrix}\right.
\end{gathered}
\label{eq:random_td_transition_eq}
\end{equation}

Consequently, we define the transition probability function for the \coolnameabbr{} $\mathbb{P}_d: \mathrm{X}_d \times \mathrm{Act} \times \mathrm{X}_d \rightarrow [0,1]$ as shown in Eq. \ref{eq:random_td_transition_eq}. In Eq. \ref{eq:random_td_transition_eq}, $y_{t-\tau_t} = \mathbb{P}(s_{t-\tau_t+1} \mid s_{t-\tau_t}, a_{t-\tau_t})$ is the one-step transition probability from the system state $s_{t-\tau_t}$ to $s_{t-\tau_t+1}$ while executing the action $a_{t-\tau_t}$. We note that the system transition probabilities in \emph{Case 2} and \emph{Case 3} can be factorized into the $\tau_t + 1$ and $\tau_t - \tau_{t+1} + 1$ terms in Eq. \ref{eq:random_td_transition_eq} respectively. Now, we prove that $\mathbb{P}_d$ is a valid conditional probability distribution with support over the state space $\mathrm{X}_d$. First we note that since the conditional distributions $\ptau, y_{t} \geq 0$, $\mathbb{P}_d \geq 0$. Next, we show that $\sum_{x_{t+1} \in \mathrm{X}_d} \mathbb{P}_d(x_{t+1} \mid x_t, a_t) = 1$. Substituting the transition probabilities from \emph{Cases 2,3} in place of the factorized terms in Eq. \ref{eq:random_td_transition_eq}, we get

\vspace{-3mm}
\begin{equation}
\begin{gathered}
\hspace{-2cm}\sum_{x_{t+1} \in \mathrm{X}_d} \mathbb{P}_d(x_{t+1} \mid x_t, a_t) = \ptau(\tau_t + 1 \mid \tau_t) +\\
\begin{matrix}

\hspace{-1cm}\ptau(0 \mid \tau_t)\sum\limits_{s_{t+1}} \mathbb{P}(s_{t+1} \mid s_{t-\tau_t}, a_{t-\tau_t}, \dots, a_t)+\\ 

\hspace{0.2cm}\sum\limits_{\tau' = 1}^{\tau_t}\ptau(\tau' \mid \tau_t) \sum\limits_{s_{t-\tau'}} \mathbb{P}(s_{t-\tau'} \mid s_{t-\tau_t}, a_{t-\tau_t}, \dots, a_{t-\tau'-1}).

\end{matrix}
\end{gathered}
\label{eq:random_td_transition_proof}
\end{equation}

The inner summations in the second and third terms of the right-hand side of Eq. \ref{eq:random_td_transition_proof} equate to 1. Hence, $\sum_{x_{t+1} \in \mathrm{X}_d} \mathbb{P}_d(x_{t+1} \mid x_t, a_t) = \sum_{\tau' \in \Omega} \ptau(\tau' \mid \tau) = 1$.

Thus, the \coolnameabbr{} $\mathcal{M}_d$ is the tuple $\NetworkedMdpTuple$. From the definition of the state space of the \coolnameabbr{}, we denote the unsafe states for the \coolnameabbr{}, $\mathrm{X_{unsafe}}$, as a subset of $\bad \times (\mathrm{Act} \cup \{ \phi \})^{\maxcommdelay} \times \Omega$, where $\bad$ is the unsafe states set for the \emph{Basic} MDP. In other words, the \coolnameabbr{} state at time $t$, $x_t = (s_{t-\tau_t}, (a_{t-\tau_t}, \dots, a_{t-1}, \phi, \dots, \phi), \tau_t)$, is unsafe if $s_{t-\tau_t} \in \bad$. Additionally, we show how to construct the \coolnameabbr{} when only $\maxcommdelay$ is known, and $\ptau$ is not. Since the delay is upper-bounded by $\maxcommdelay$, the action corresponding to the observation $s_{t-\maxcommdelay}$ is always available at timestep $t$. Therefore, we take $s_{t-\maxcommdelay}$ as the latest available system state and consider the delay to be a constant and equal to $\maxcommdelay$. Consequently, the initial delay is set to $\maxcommdelay$ and the action buffer is set to $\left\{(a_s, a_s,...,a_s)\right\}$, where $a_s$ is the action that does not affect the agent's state. In our hardware setup, $a_s$ is the ego-velocity of 0 m/s.

\subsection{Shield Design for Safe Networked Control}
\label{sec:guaranteed_shield}

In this section, we show how to construct a shield for any MDP $\mathcal{M} = \mdptuple$, a specification $\varphi = \square \neg \bad$ and a policy $\pi$. The shield should ensure that when $\pi$ is executed in the presence of the shield, the initial state distribution should satisfy $\varphi$ with at least the desired safety probability $\delta$. First, we formally define the shield.

\begin{definition}The $\epshield$, $\epshieldsym: \mathrm{S} \rightarrow 2^{\mathrm{Act}}$, for any state $s \in \mathrm{S}$, and $\epsilon \in [0,1]$ is
\end{definition}
\vspace{-4mm}
\begin{equation}
    \epshieldsym(s)=\left\{\begin{matrix}
\{ a \mid \maxqphi \geq \epsilon \} & \mathrm{if} \; \maxprphi(s) \geq \epsilon,\\ 
\{\operatorname*{argmax}_a \maxqphi\} & \mathrm{if} \; \maxprphi(s) < \epsilon.
\end{matrix}\right.
\label{eq:eps_safe_eqn}
\end{equation}

During run-time, the action executed is different from $\pi(s)$ only if $\pi(s) \not \in \epshieldsym(s)$; in which case an action from $\epshieldsym(s)$ is chosen. Hence, the $\epshield$ is \emph{minimally intrusive}. Now, we show there exists $\epsilon$ such that the run-time monitoring of $\pi$ by the $\epshield$ $\epshieldsym$ provides a safety probability greater than or equal to $\delta$ for the initial states.

\begin{definition}The modified policy, $\pi_\epsilon: \mathrm{S} \rightarrow \mathrm{Act}$, for any state $s \in \mathrm{S}$, policy $\pi$, and $\epshield$ $\epshieldsym$ is
\end{definition}
\vspace{-4mm}
\begin{equation}
    \pi_\epsilon(s)=\left\{\begin{matrix}
 \pi(s) & \mathrm{if} \; \pi(s) \in \epshieldsym(s),\\ 
\text{pick from } \epshieldsym(s) & \mathrm{if} \; \pi(s) \not \in \epshieldsym(s).
\end{matrix}\right.
\label{eq:pi_epsilon_defn}
\end{equation}
Observe that the \emph{modified policy} $\pi_\epsilon$ is a result of the run-time monitoring of $\pi$ by the $\epshield$ $\epshieldsym$. In other words, $\pi_\epsilon$ is the policy that is executed during networked control.

\begin{proposition}
The safety probability for a state $s \in \stateset$ while executing $\pi_\epsilon$, $\piprphieps(s)$, is lower bounded by $\minprphieps(s)$ where the MDP $\mathcal{M}_\epsilon = \epsmdptuple$. The lower bound $\minprphieps(s)$ is a non-decreasing function of $\epsilon$.
\label{proposition:lower_bound}
\end{proposition}
\begin{proof}First, we note that executing $\pi_{\epsilon}$ for $\mathcal{M}$ is equivalent to executing $\pi$ for $\mathcal{M}_\epsilon = \epsmdptuple$, where the allowed action set for each state $s$ is given by $\mathrm{C}_\epsilon(s)$. Hence, the minimum safety probability for $\mathcal{M}_\epsilon$ denoted by $\minprphieps(s)$ is the lower bound for $\piprphieps(s)$. Now, we show by contradiction that for $\epsilon, \bar{\epsilon} \in [0,1]$ and $\epsilon < \bar{\epsilon}$, $\minprphieps(s) \leq \minprphiepsbar(s)$ for any state $s \in \stateset$. Assume for the two MDPs, $\mathcal{M}_\epsilon$ and $\mathcal{M}_{\bar{\epsilon}}$, $\minprphieps(s) > \minprphiepsbar(s)$ for some state $s \in \stateset$. This implies that the policy that corresponds to $\minprphiepsbar(s)$ does not exist for $\mathcal{M}_\epsilon$, and hence $\mathrm{C}_\epsilon(s') \subset \mathrm{C}_{\bar{\epsilon}}(s')$ for some $s' \in \stateset$. But,  from the definition of $\epshield$, if $\epsilon < \bar{\epsilon}$, then $\mathrm{C}_\epsilon(s) \supseteq \mathrm{C}_{\bar{\epsilon}}(s) \ \ \forall s \in \stateset$, which is a contradiction. Thus, the lower bound on  $\piprphieps(s)$ is a \emph{non-decreasing} function of $\epsilon$. 
\end{proof}

\begin{remark}For the MDP $\mathcal{M} = \mdptuple$ and the safety property $\varphi$, note that the safety probability for any $s \in \stateset$ is upper-bounded by $\maxprphi(s)$. So, for the initial state distribution, the upper bound on the safety probability is $\maxinitsafety$. Hence, any choice of the desired safety probability $\delta$ should satisfy $\delta \leq \maxinitsafety$.
\label{remark:safety_upper_bound}
\end{remark}

\begin{algorithm}
	\caption{Shield Design}
    \textbf{Input:} MDP $\mathcal{M} = \mdptuple$, policy $\pi$, safety specification $\varphi=\square \neg \bad$, desired safety probability $\delta$ \\
    \textbf{Output:} $\epshield$, $\epstarshieldsym$. 
	\begin{algorithmic}[1]
        \State Initialize $\epshield, \epstarshieldsym(s) = \mathbb{A}(s) \ \ \forall s \in \stateset.$ 
        \State Compute $\maxqphi$ for all state-action pairs in $\mathcal{M}$.
        \For {$\epsilon \leftarrow [0, \eta, 2\eta, \dots,1]$} 
            \State Determine $\epshieldsym(s)$ for each state $s \in \stateset$ as in Eq. \ref{eq:eps_safe_eqn} 
            \State Determine the \emph{modified policy} $\pi_\epsilon$ as in Eq. \ref{eq:pi_epsilon_defn}.
            \State Compute $\piprphieps(s)$ for all states in $\mathrm{S}$ as in Sec. \ref{sec:background}.
            \If{$\actinitsafetyeps \geq \delta$}
                 $\epstarshieldsym = \epshieldsym$
                 \State \textbf{break}
            \EndIf
        \EndFor
        \State \textbf{return} $\epstarshieldsym$.
    \end{algorithmic}       
\label{alg:shield_design}
\end{algorithm}

The Algorithm \ref{alg:shield_design} takes as input the MDP $\mathcal{M}$, policy $\pi$, specification $\varphi$, and desired safety probability $\delta$, and outputs the synthesized $\epshield$ $\epstarshieldsym$. In line 2, we compute the maximum safety probability for all state-action pairs in $\mathcal{M}$, as explained in Sec. \ref{sec:background}, Eq. \ref{eq:qmax_computation}. Then, we gradually (based on the granularity $\eta$) vary the parameter $\epsilon$ from 0 to 1 until the safety probability for the initial state distribution, while executing the \emph{modified policy} $\pi_\epsilon$, is greater than or equal to the desired safety probability $\delta$ (lines 3-10).

\begin{theorem} (Termination with guaranteed safety). For a given MDP $\mathcal{M}$ and a policy $\pi$, Algorithm \ref{alg:shield_design} always terminates with an $\epshield$, $\epstarshieldsym$ as in Eq. \ref{eq:eps_safe_eqn}, such that the modified policy $\pi_{\epsilon}$, a combination of $\pi$ and $\epstarshieldsym$ (Eq. \ref{eq:pi_epsilon_defn}), satisfies the safety property $\varphi = \square \neg \bad$ for the initial state distribution with a probability greater than or equal to the desired safety probability $\delta$, where $\delta \leq \maxinitsafety$. \label{theorem:guaranteed_safety}\end{theorem}

\begin{proof} From Proposition \ref{proposition:lower_bound}, we know that $\minprphieps(s)$ is a \emph{non-decreasing} function of $\epsilon \ \forall s \in \stateset$. For $\epsilon = 0$, note that $\epshieldsym(s) = \mathbb{A}(s)$, and therefore $\minprphieps(s) = \minprphi(s)$. Moreover, for $\epsilon = 1$, note that $\epshieldsym(s) = \left \{ \operatorname*{argmax}_a \maxqphi \right \}$ from Eq. \ref{eq:eps_safe_eqn}. This implies $\pi_\epsilon$ is the same as the optimally safe policy, $\optsafepolicy$, from Sec. \ref{sec:background}. Consequently, we have $\minprphieps(s) = \maxprphi(s)$. To summarize, $\minprphieps(s)$ is a non-decreasing function of $\epsilon$ that lies between $\minprphi(s)$ and $\maxprphi(s)$. 

Since expectation is a linear operation, $\mininitsafetyeps$ is also a non-decreasing function of $\epsilon$ that lies between $\mininitsafety$ and $\maxinitsafety$. Therefore, for any desired safety probability $\delta \leq \maxinitsafety$ (from Remark \ref{remark:safety_upper_bound}), there exists an $\epsilon \in [0, 1]$ such that $\mininitsafetyeps \geq \delta$. Finally, since $\actinitsafetyeps$ is lower bounded by $\mininitsafetyeps$ (Proposition \ref{proposition:lower_bound}), we conclude that the Algorithm \ref{alg:shield_design} always terminates with the $\epshield, \epstarshieldsym$, that guarantees the desired safety probaility $\delta$.
\end{proof}

We note that irrespective of the choice to pick any action from $\epshieldsym(s)$ when $\pi(s) \not \in \epshieldsym(s)$ (Eq. \ref{eq:pi_epsilon_defn}), Algorithm \ref{alg:shield_design} returns an $\epshield, \epstarshieldsym$, which guarantees the desired safety probability. For example, one could select actions from $\epshieldsym(s)$ prioritizing either task-efficiency or safety ($\operatorname*{argmax}_a \maxqphi$). We also note that the shield design in \cite{jansen2020safe} and \cite{bouton2019reinforcement} is similar to our $\epshield$ definition. However, our key novelty is that unlike \cite{jansen2020safe} and \cite{bouton2019reinforcement} which do not provide any guarantee on achieving the required safety probability, our approach (Algorithm \ref{alg:shield_design}) returns $\epstarshieldsym$ which guarantees the required safety probability.

\begin{remark}Algorithm \ref{alg:shield_design} can be modified to yield an $\epshield$ even when $\pi$ is not known, in cases like human-teleoperation. Since the modified policy cannot be computed without $\pi$, we instead check for $\mininitsafetyeps \geq \delta$ in line 7 of the Algorithm \ref{alg:shield_design}. This guarantees safety probability of at least $\delta$ for any modified policy $\pi_\epsilon$.\end{remark}

Hence, for safe networked control, we construct the \coolnameabbr{} (refer to Sec. \ref{sec:dcmdp}), and given the safety specification $\varphi = \square \neg \bad$ and the cloud controller $\picloud$, we use Algorithm \ref{alg:shield_design} to construct the $\epshield, \epstarshieldsym$, which ensures a safety probability greater than or equal to $\delta$ for the networked control system. The same can be extended to reach-avoid specifications by casting them into reachability specifications (refer to \cite{baier2008principles}). More broadly, our approach works for any specification that can be cast into a reachability specification.

\section{Experiments}
\label{sec:experiments}

Now, we show empirically that the shield ensures safety in the presence of communication delays. We analyze the behavior of the agent with shields constructed using two different \coolnameabbr s:  ``constant delay'' when only $\maxcommdelay$ is known, and ``random delay'' when $\ptau$ is modeled additionally. 

We test on three environments,

\begin{itemize}[leftmargin=*]
    \item A 2D $8 \times 8$ gridworld simulation setup where the controlled robot, initialized at (0,0), is tasked with reaching the $\mathrm{goal}$ at (7,7) while avoiding collision with a dynamic obstacle. Each episode runs for 50 timesteps. An episode is considered a \emph{win} if the robot reaches the goal without colliding, a \emph{loss} if there is a collision or a \emph{draw} otherwise. The cloud controller is learned using tabular Q-learning.  
    \item A car-following simulation setup where the ego robot has to follow the leader car with a minimum safety distance of $5$m. The system state consists of relative distance and relative velocity. The leader car can accelerate anywhere between $-0.2$m/s$^2$ and $0.2$m/s$^2$, and the ego robot can accelerate between $-0.5$m/s$^2$ and $0.5$m/s$^2$. Each episode runs for $100$s. The cloud controller is a pre-trained RL agent that maximizes distance traveled and minimizes collisions with the leader. We discretize the relative distance and relative velocity to obtain a finite state space.

    \item The hardware setup (Fig. \ref{fig:introduction_figure}) with two F1/10th cars \cite{o2019f1}. The ego robot is equipped with a laser rangefinder. The generated point cloud is transmitted over WiFi to a remote server (the cloud). Here the state is estimated and a time-optimal control command is sent back over WiFi to the robot. The robot has to follow the leader as quickly as possible while maintaining a safe distance of at least 0.2m.

\end{itemize}

\begin{figure*}[thpb]
\centering

\begin{subfigure}[b]{0.4\columnwidth}
     \includegraphics[width=1.0\textwidth]{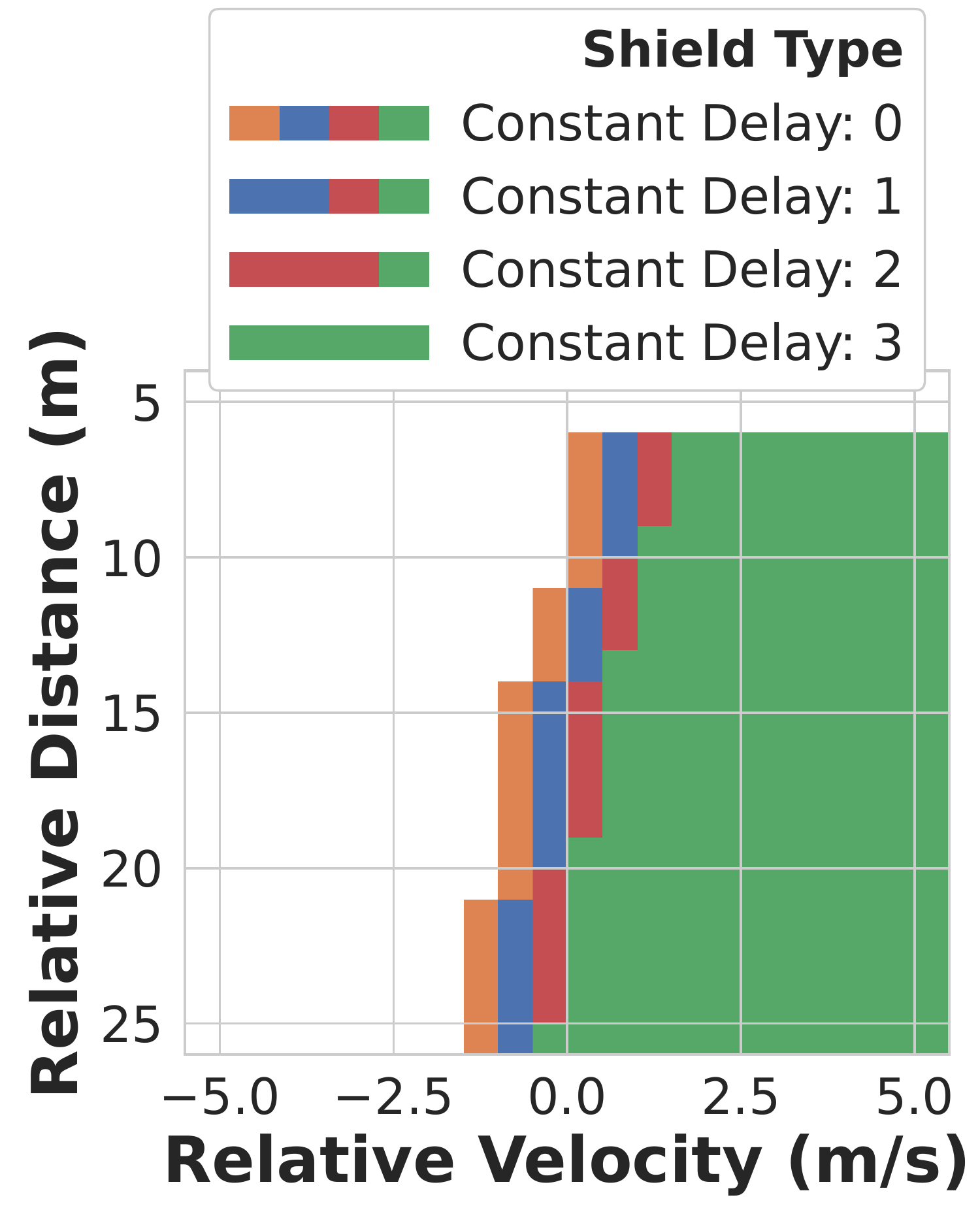}
    \caption{\footnotesize{Safe states - constant delay}}
    \label{fig:simulation_results_a}
 \end{subfigure}
 \hfill
 \begin{subfigure}[b]{0.4\columnwidth}
     \includegraphics[width=1.0\textwidth]{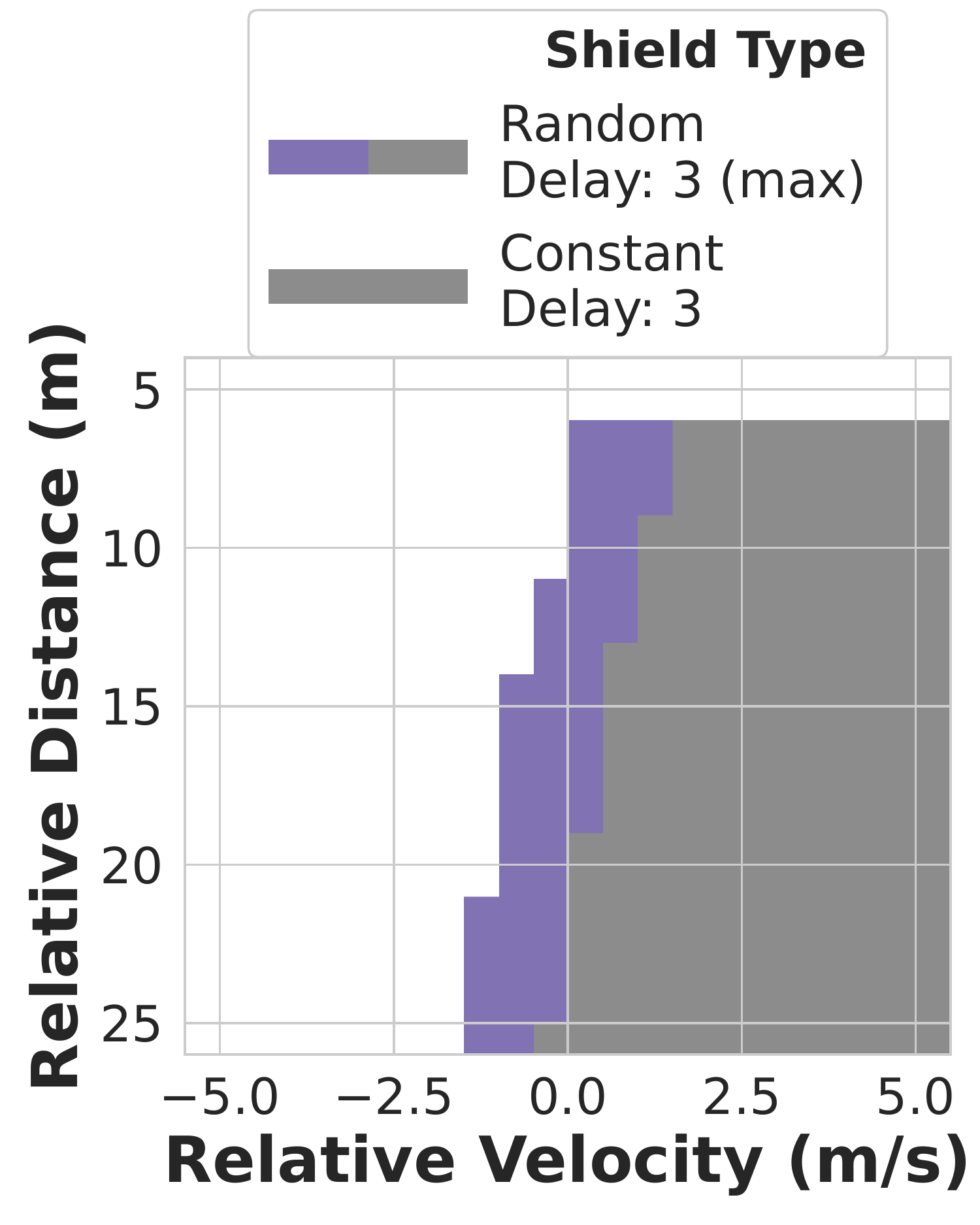}
    \caption{\footnotesize{Safe states - random delay}}
    \label{fig:simulation_results_b}
 \end{subfigure}
  \hfill
 \begin{subfigure}[b]{0.9\columnwidth}
     \centering
     \includegraphics[width=1.0\textwidth]{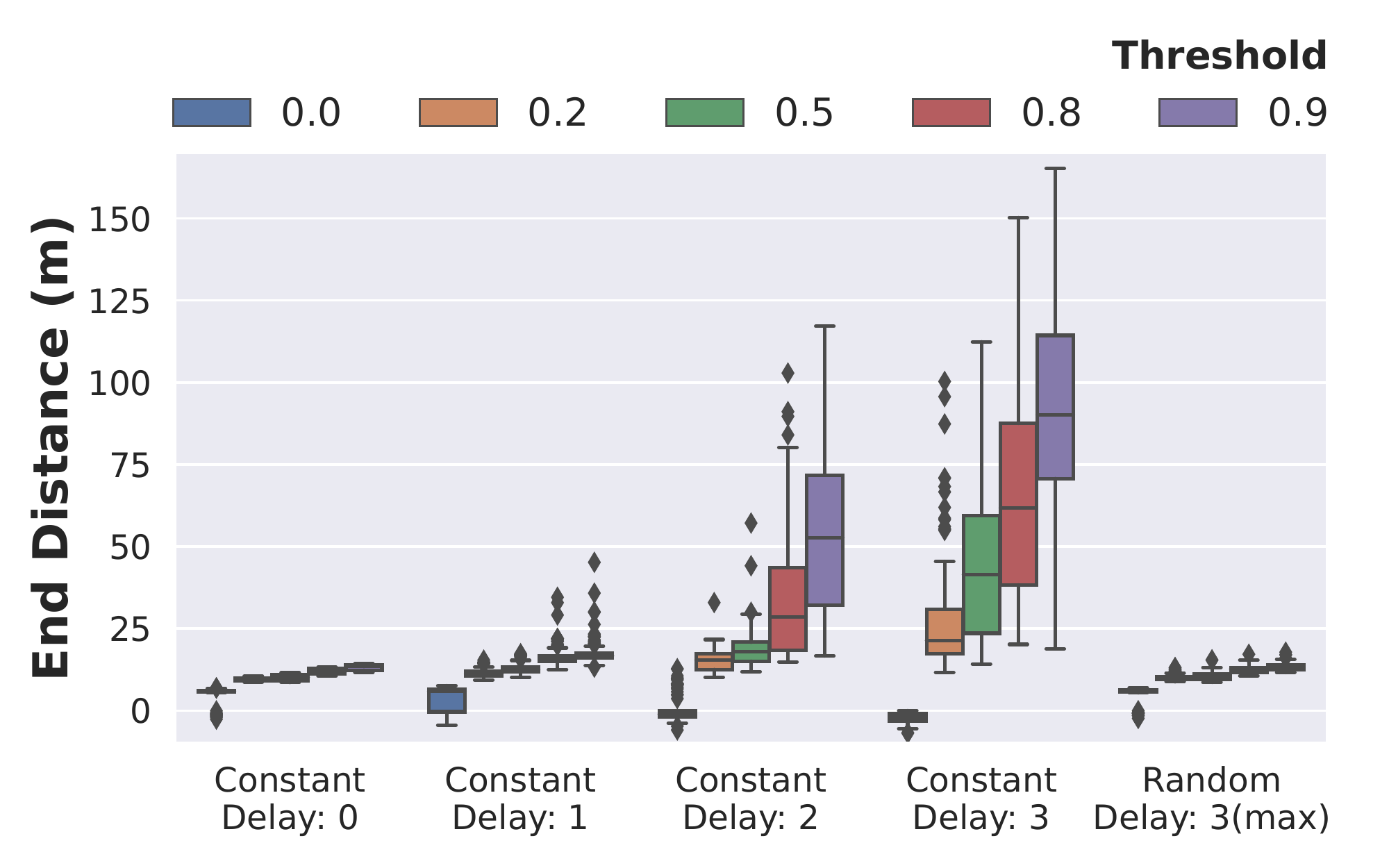}
    \caption{\footnotesize{Quantitative results}}
    \label{fig:simulation_results_c}
 \end{subfigure}
 \begin{subfigure}[b]{0.4\columnwidth}
     \centering
     \includegraphics[width=1.0\textwidth]{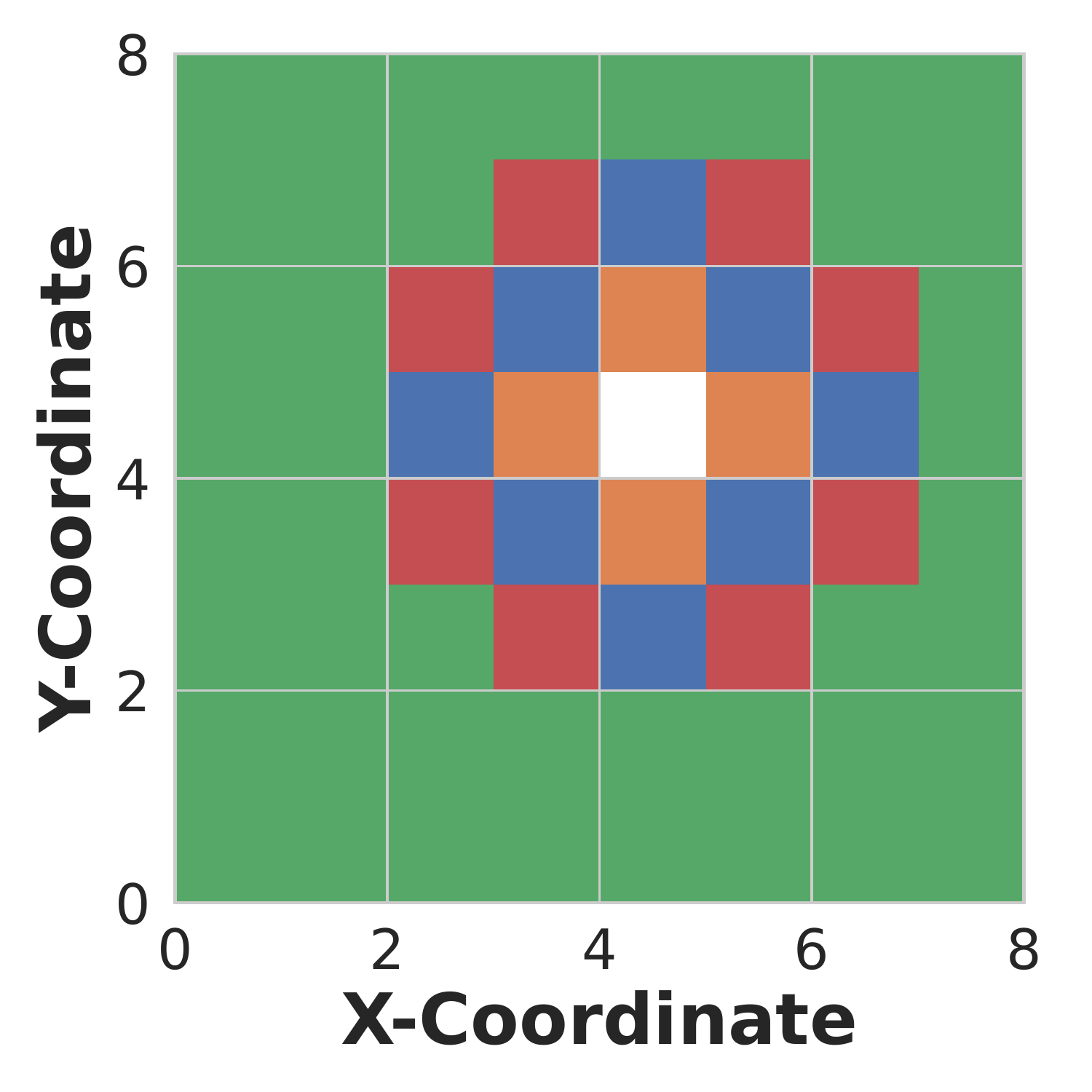}
    \caption{\footnotesize{Safe states - constant delay}}
    \label{fig:simulation_results_d}
 \end{subfigure}
 \hfill
 \begin{subfigure}[b]{0.4\columnwidth}
     \centering
     \includegraphics[width=1.0\textwidth]{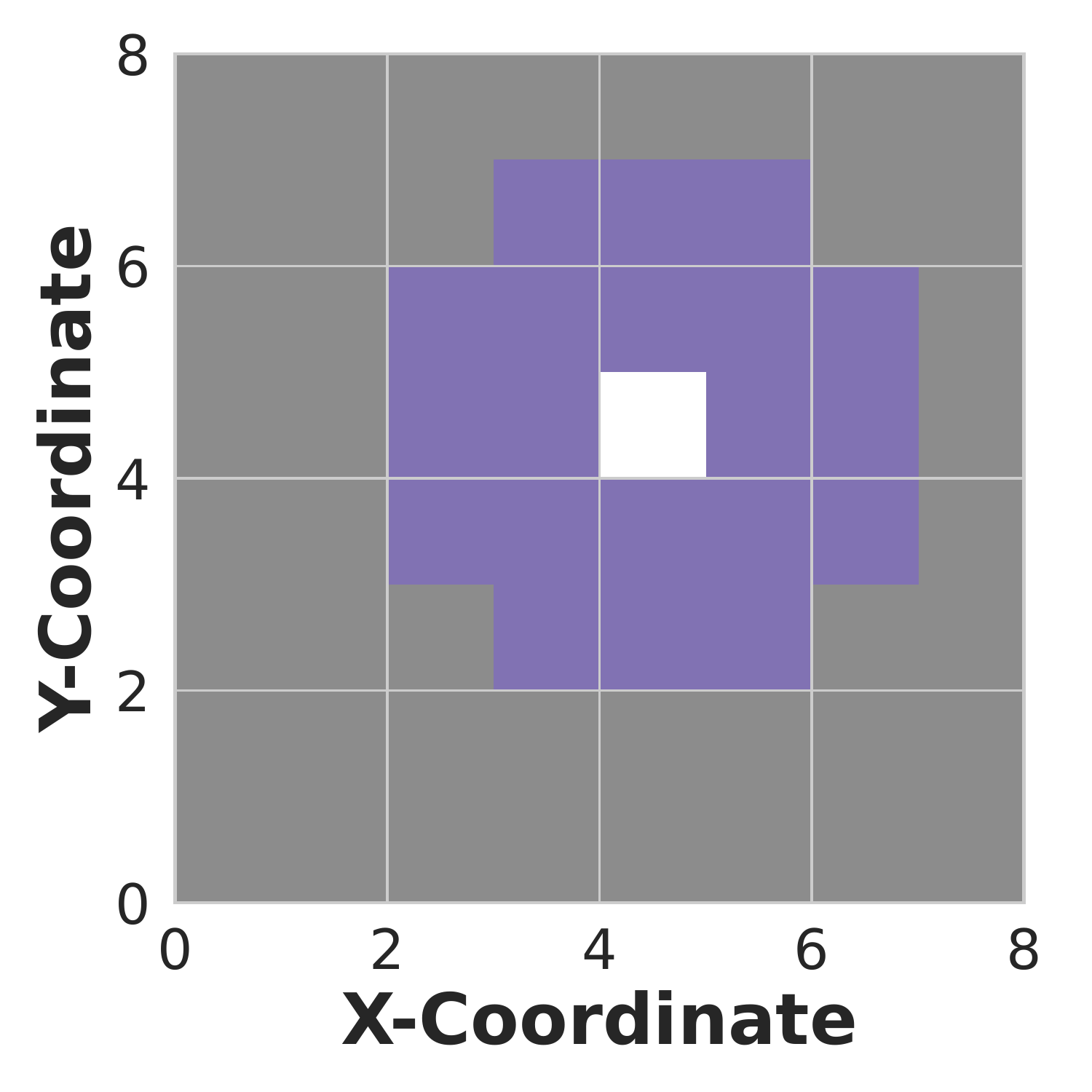}
    \caption{\footnotesize{Safe states - random delay}}
    \label{fig:simulation_results_e}
 \end{subfigure}
 \hfill
 \begin{subfigure}[b]{0.9\columnwidth}
     \centering
     \includegraphics[width=1.0\textwidth]{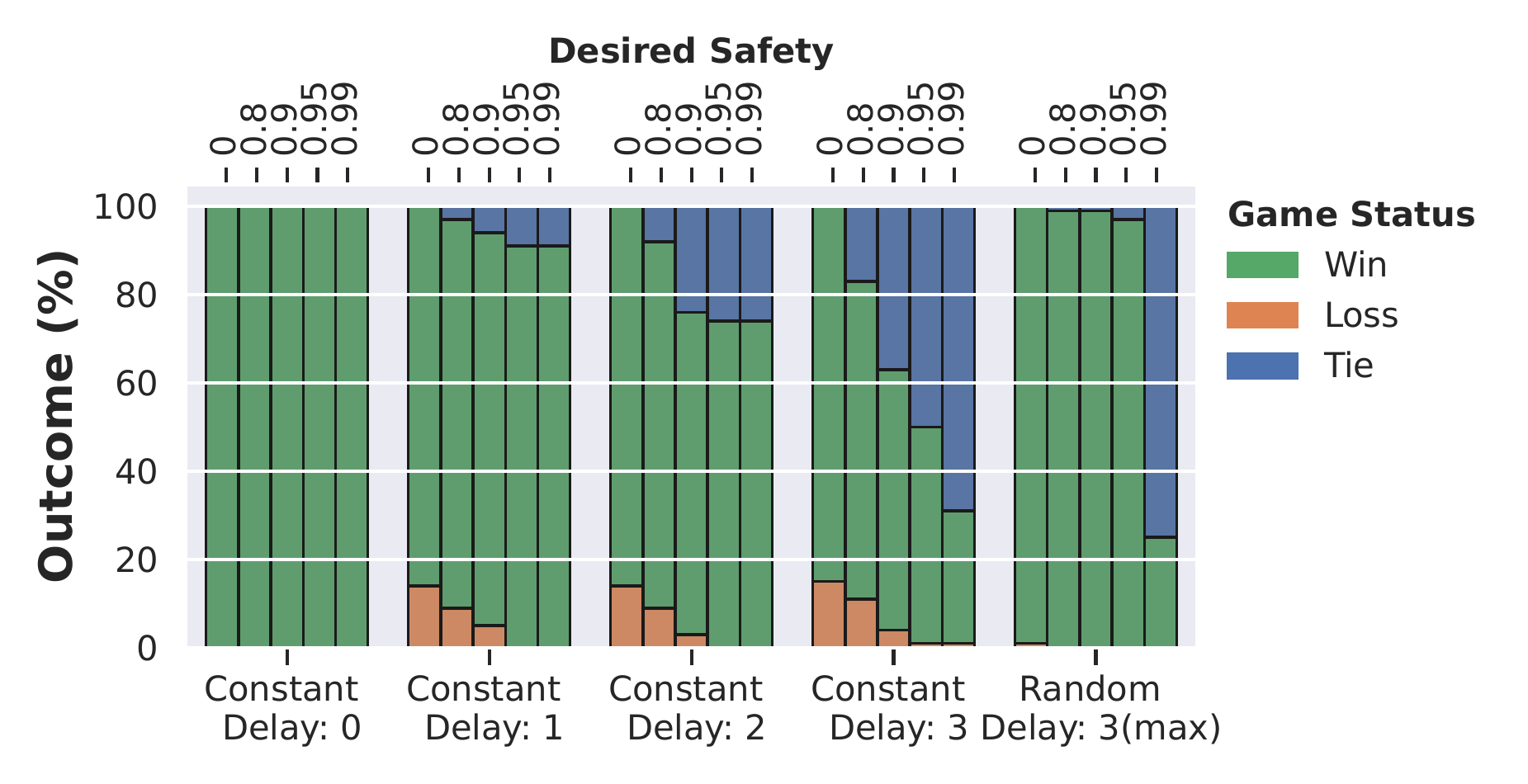}
    \caption{\footnotesize{Quantitative results}}
    \label{fig:simulation_results_f}
 \end{subfigure}
     
\caption{\textbf{Shielding leads to safe networked control in simulations.} Top row - Car following simulation results; Bottom row - Gridworld simulation results. Figs. \ref{fig:simulation_results_a} and \ref{fig:simulation_results_d} show the set of safe initial states with maximum safety probability greater than $0.95$ for the \coolname{} for the constant delay case. The set of safe states expands as the maximum delay ($\maxcommdelay$) decreases. This is depicted by the legend that has multiple colors attributed to lower latencies.  Figs. \ref{fig:simulation_results_b} and \ref{fig:simulation_results_e} compares the set of safe initial states with maximum safety probability greater than $0.95$ for the random and constant delay cases with $\maxcommdelay=3$ in both the cases. The set of safe states is larger in the case of random delay as the shield exploits the knowledge of the delay transitions to allow the agent to act more aggressively. The white color represents initial states that have maximum safety probability less than $0.95$. For the gridworld setup, the obstacle is located at (4,4). Figs. \ref{fig:simulation_results_c} and \ref{fig:simulation_results_f} show that as latency increases, the system tends to be conservative, leading to increased distances in the car-following scenario, and an increased number of ties in the gridworld case.}

\label{fig:simulation_results}
\vspace{-1em}
\end{figure*}

\begin{figure*}[h!]
\vskip 0.2in
\centering
\begin{subfigure}[b]{0.9\columnwidth}
     \centering
     \includegraphics[width=1.0\textwidth]{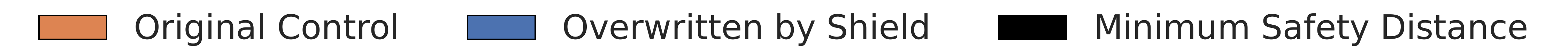}
     \includegraphics[width=1.0\textwidth]{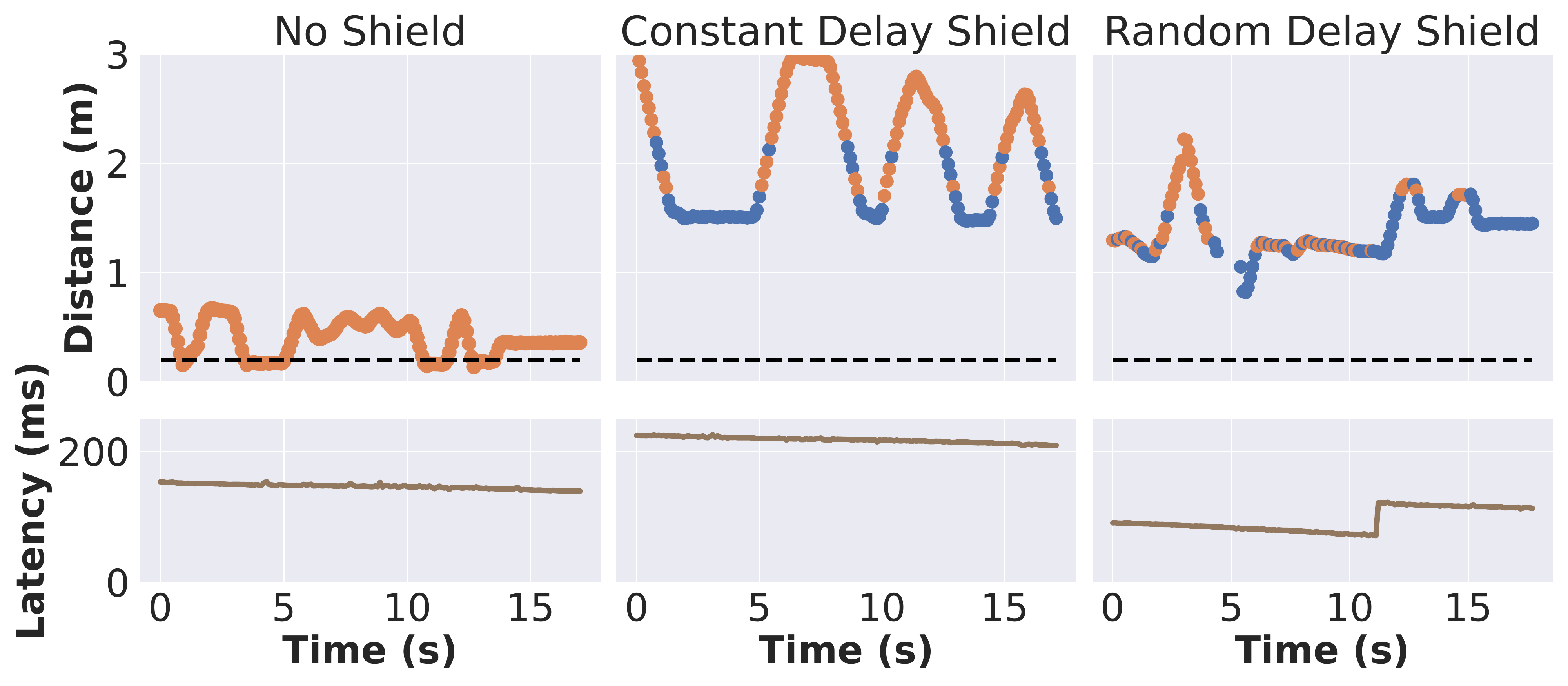}
    \caption{\footnotesize{Qualitative Results.}}
    \label{fig:demo_results_qualitative}
 \end{subfigure}
 \hfill
 \begin{subfigure}[b]{0.35\columnwidth}
     \centering
     \includegraphics[width=1.0\textwidth]{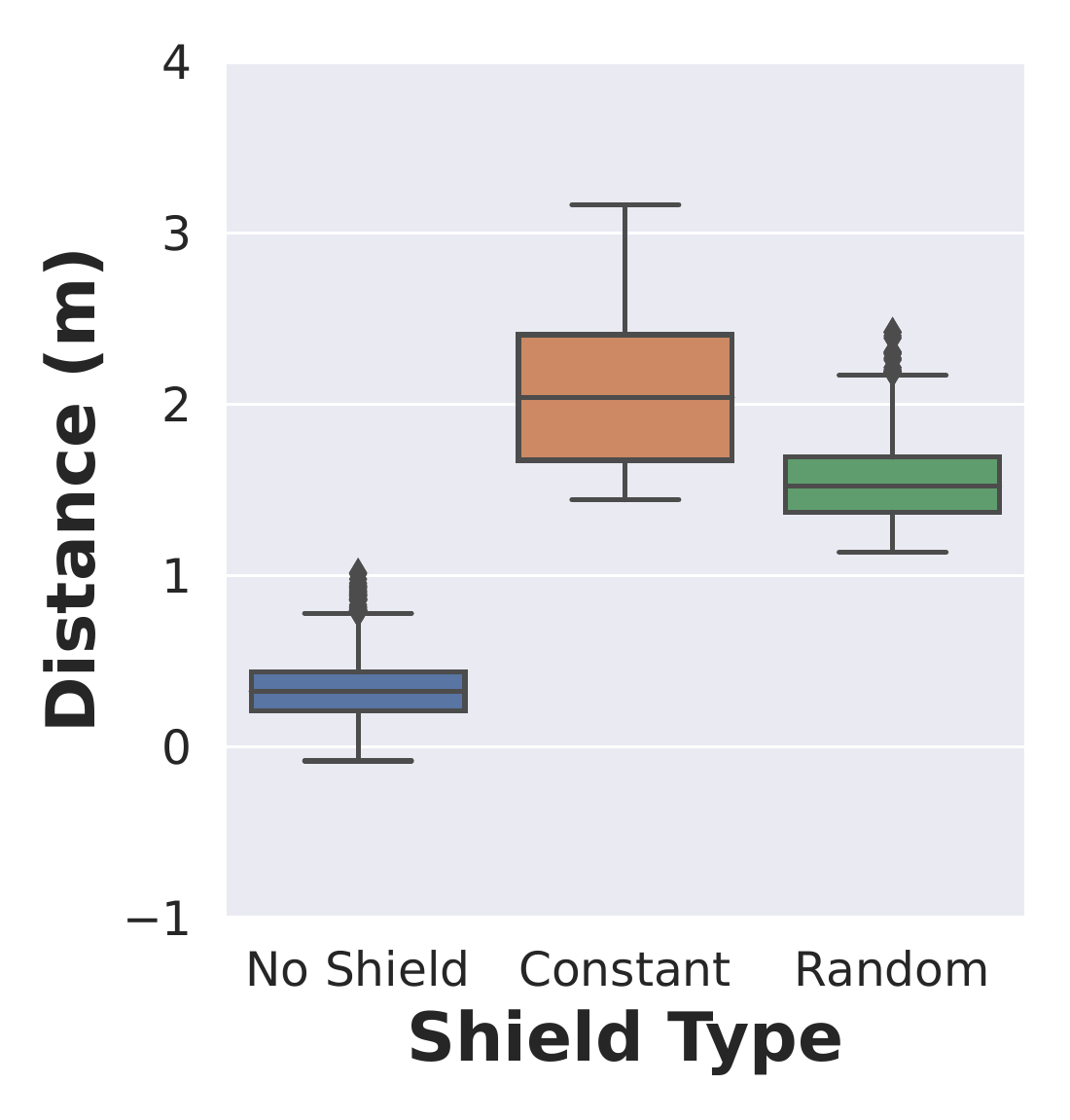}
     \caption{\footnotesize{Quantitative Results.}}
     \label{fig:demo_results_quantitative}
 \end{subfigure}
 \hfill
 \begin{subfigure}[b]{0.375\columnwidth}
     \centering
     \includegraphics[width=1.0\textwidth]{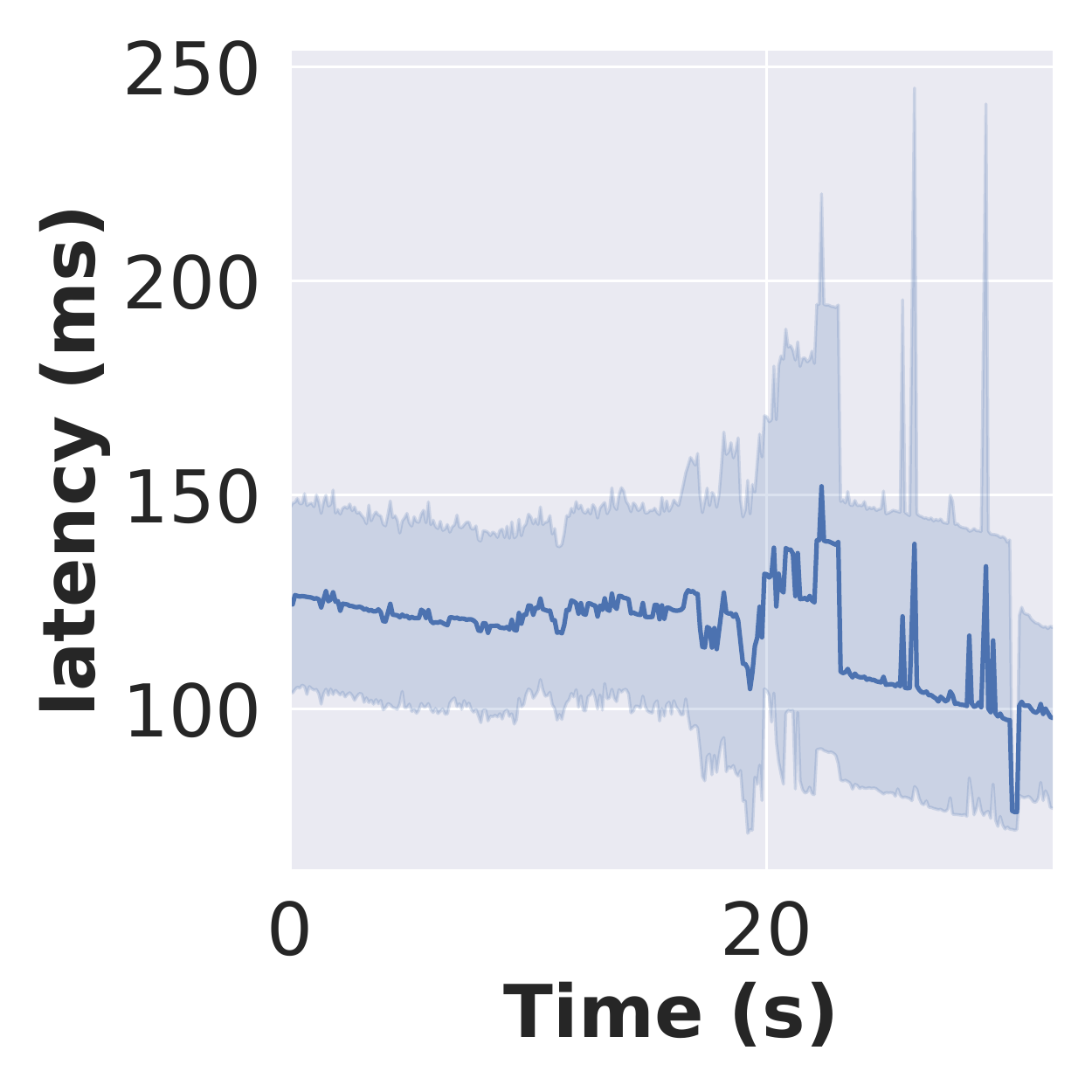}
     \caption{\footnotesize{WiFi traces.}}
     \label{fig:wifi_traces}
 \end{subfigure}
 \hfill
\begin{subfigure}[b]{0.375\columnwidth}
     \centering
     \includegraphics[width=1.0\textwidth]{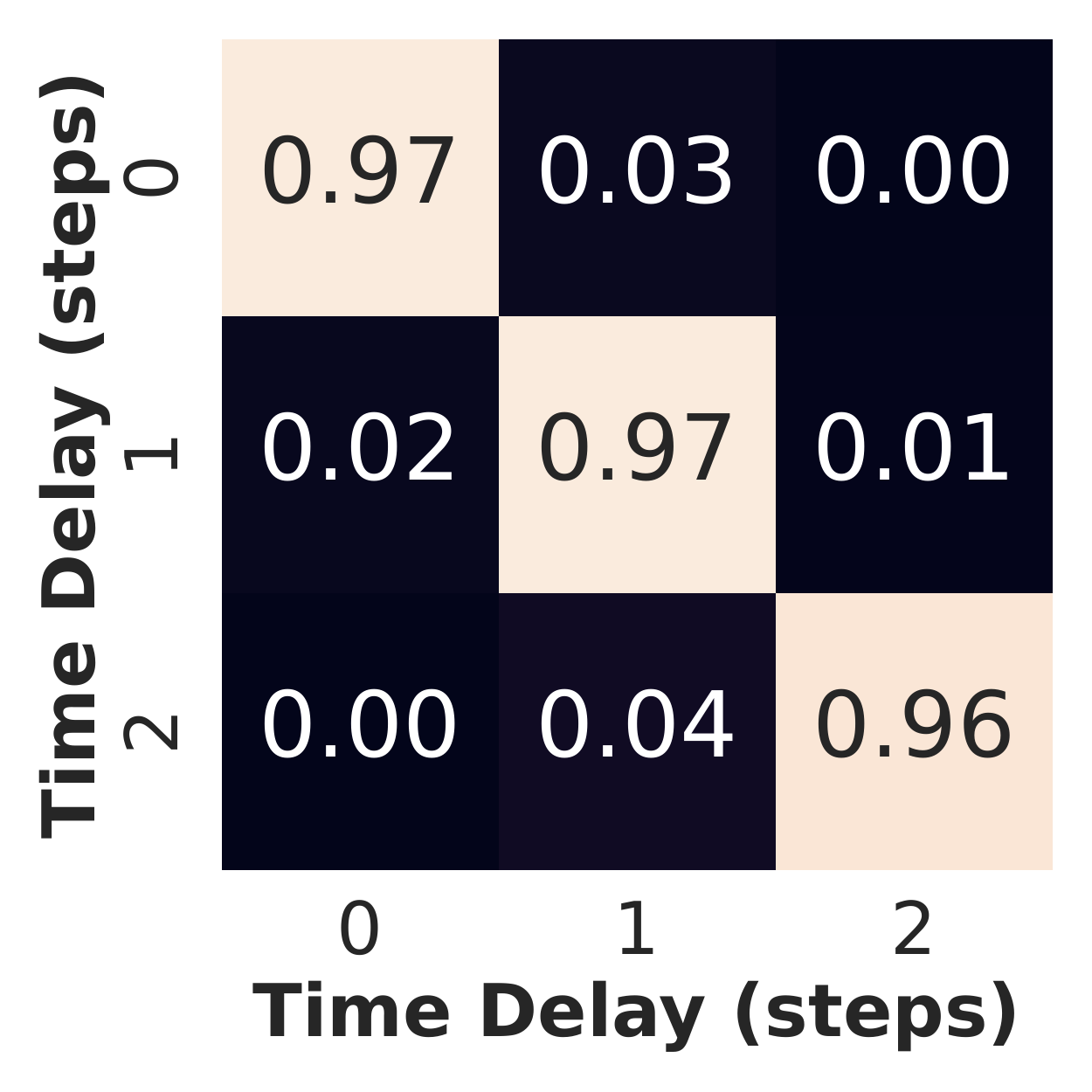}
     \caption{\footnotesize{Delay transition model}}
     \label{fig:delay_transition_matrix}
 \end{subfigure}
     
\caption{\textbf{Real World Demonstration Results.} Fig. \ref{fig:demo_results_qualitative} shows recorded trajectories from our hardware setup. Without our safe networked control approach, the system fails to satisfy the safety specification ``always remain at least $0.2$ meters away''. However, our approach can satisfy the safety specification for constant and random delays. Also, as seen in Fig. \ref{fig:demo_results_quantitative}, the shield for the random delay case exploits the knowledge of delay transitions in Wi-Fi rather than assuming only the maximum latency, which allows the ego robot to follow the leader car at a closer distance. We set $\epsilon=0.95$ to construct the $\epshield$. In Fig. \ref{fig:wifi_traces} and Fig. \ref{fig:delay_transition_matrix}, we show how the delay transition probability function, $\ptau$, is estimated using multiple runs of the Wi-Fi latency time-series data. Here, $\ptau$ is a conditional probability distribution with $3$ possible delays ($0,1,2$), where each delay is a bin of size $100$ ms.}

\label{fig:demo_results}
\vspace{-1em}
\end{figure*}

For the car-following and the hardware setup, the safety specification is $\square \neg \bad$, where $\bad$ consists of states where the distance between the cars is less than 5m and 0.2m respectively. For the gridworld, it is $\neg \ \bad \ \cup \mathrm{goal}$, where $\bad$ is the set of states where the robot and the obstacle are in the same location, i.e., collision. In our simulation environments, we experimented with the maximum delay ranging from 0 to 3 time steps for both constant and random delay. For random delay, we assume a delay transition probability function, $\latencymodel$, where the delay is mostly 0 and changes to other values with low probability.

\emph{How does the performance of our safe networked control approach vary with communication delay?}
The safety of the teleoperated robot reduces when the communication delay increases. We observe this in Figs. \ref{fig:simulation_results_a} and \ref{fig:simulation_results_d} for the two simulation setups, for the constant delay case. The set of states for which maximum safety probability $\maxprphi(s)$ (Sec. \ref{sec:background}) is greater than a $\delta$ value shrinks with increasing delay. It shows that when the delay is large it is safer for the robot to stay farther away from the dynamic obstacle (gridworld) and for the ego robot to maintain a larger relative distance and velocity between itself and the leader car (car-following). 

The shields ensure the desired safety probability $\delta$ for different delays. However, for the same safety probability, the task performance degrades with increasing delay due to increasing uncertainty in the system state. We show this quantitatively for the two simulation setups with constant delay. In the gridworld, with larger delays, the shield increasingly restricts the robot from moving aggressively toward the $\mathrm{goal}$ to avoid collisions. As such, it effectively sacrifices a win for a tie. Similarly, in the car-following scenario, the distance maintained from the leader robot increases (Fig. \ref{fig:simulation_results_c}). We observe a similar trend in our hardware setup that runs on a wireless network with stochastic delays (Fig. \ref{fig:demo_results_qualitative}). During the initial $10$s when the delay is less than $100$ ms, the average distance maintained is less than $1.25$ m. Then, when the delay is about $200$ ms, the ego-robot starts to maintain a larger distance of around $1.5$ m. 

\emph{How does $\delta$ affect the safety-efficiency trade-off?}

Our key insight is that we can vary $\delta$ to trade off safety for task efficiency. We observe this in the constant delay case, where increasing $\delta$ leads to increasingly conservative behavior with more restrictions from the shield. In the gridworld, Fig. \ref{fig:simulation_results_f} shows fewer wins and more ties, an indicator of reduced task efficiency. Additionally, the number of losses decreases as safety is prioritized. Similarly, for car-following, the average distance maintained from the leader car increases (Fig. \ref{fig:simulation_results_c}). 
On the other hand, $\delta=0$ is the un-shielded approach, which leads to a violation of the safety specification.

\emph{How does incorporating the delay transition probability function $\latencymodel$ affect safety and efficiency?} We now illustrate that by incorporating $\latencymodel$, our safe networked control approach performs more efficiently since the \coolnameabbr{} model is more accurate. Whereas, when only $\maxcommdelay$ is known, the model is less accurate as it assumes the observation from $\maxcommdelay$ steps before to be the latest available system state even if the delay is small and more recent system states are available. We compare the \coolnameabbr{} for constant delay against the \coolnameabbr{} for random delay with $\maxcommdelay = 3$ in both cases. Firstly, Figs. \ref{fig:simulation_results_b} and \ref{fig:simulation_results_e} show that the set of states for which maximum safety probability $\maxprphi(s)$ (Sec. \ref{sec:background}) is greater than a $\delta$ value is larger when $\latencymodel$ is incorporated. Secondly, we observe more wins and fewer draws in the gridworld, and lower aggregate distance maintained for the car following setup as seen in Figs. \ref{fig:simulation_results_f} and \ref{fig:simulation_results_c} respectively. For the hardware setup, $\latencymodel$ is obtained experimentally (see Fig. \ref{fig:delay_transition_matrix}). Similar to the car following setup, the distance maintained between the two robots is less in the case of random delay when compared to constant delay (see Fig. \ref{fig:demo_results_quantitative}). This difference in safety distance is statistically significant with a Wilcoxon p-value $ < 0.001$. 
To summarize, we infer that incorporating $\ptau$ in our \coolnameabbr{} design allows for efficient task performance \emph{without} compromising safety.

\emph{Does a minimally intrusive shield always lead to safety?} Fig. \ref{fig:demo_results_qualitative} shows the state trajectory in the presence and absence of the $\epshield$, and the instances when the $\epshield$ overwrites the cloud controller of the hardware setup. The $\epshield$ overwrites control commands when close to the leader car (relatively unsafe), and is inactive when further away. For example, in the constant delay case, the shield is inactive when the distance is above $\sim$2.25m, and still ensures safety.

\begin{table}[]
\resizebox{\columnwidth}{!}{%
\begin{tabular}{@{}lllllll@{}}
\toprule
 & Metrics & \begin{tabular}[c]{@{}l@{}}Constant\\ delay: 0\end{tabular} & \begin{tabular}[c]{@{}l@{}}Constant\\ delay: 1\end{tabular} & \begin{tabular}[c]{@{}l@{}}Constant\\ delay: 2\end{tabular} & \begin{tabular}[c]{@{}l@{}}Constant\\ delay: 3\end{tabular} & \begin{tabular}[c]{@{}l@{}}Random\\ delay: 3 (max)\end{tabular} \\ \midrule
\multirow{3}{*}{Car-following} & States & 484 & 2420 & 12100 & 60500 & 75504 \\ \cmidrule(l){2-7} 
 & Time (s) & 0.015 & 0.044 & 0.299 & 2.173 & 34.08 \\ \cmidrule(l){2-7} 
 & Memory (KB) & 12 & 68 & 360 & 1978 & 2645 \\ \midrule
\multirow{3}{*}{Gridworld} & States & 8192 & 40960 & 204800 & 1024000 & 1277952 \\ \cmidrule(l){2-7} 
 & Time (s) & 1.694 & 19.711 & 175.84 & 1424.44 & 3956.81 \\ \cmidrule(l){2-7} 
 & Memory (KB) & 440 & 1448 & 6808 & 36210 & 48546 \\ \bottomrule
\end{tabular}%
}
\caption{\textbf{Run time and memory analysis for the shield construction.} In this table, we show the number of states, time taken to compute the maximum safety probabilities (value iteration, line 2, Algorithm \ref{alg:shield_design}), and the memory occupied by the shield for both the car-following and the grid-world simulation environments. The value iteration process terminates when the maximum change in the safety probability for any state between two consecutive iterations is less than $10^{-6}$. 
}
\vspace{-5mm}
\label{tab:simulation_complexity}
\end{table}

\emph{What are the practical effects of discretizing the state space and communication delay?} 
We explain the effects of discretization on the time taken for the \coolnameabbr{}'s shield construction and the achievable safety probabilities. To assess the effect on the time taken for the shield construction, we quantify the time complexity of Algorithm \ref{alg:shield_design}, which mainly depends on line 2 (maximum safety probability for all state-action pairs). Since this is a value iteration procedure, the time complexity of Algorithm \ref{alg:shield_design} is in the order of $\mathcal{O}(|\stateset|^2|\mathrm{Act}|)$ for any MDP $\mathcal{M} = \mdptuple$ (refer to \cite{baier2008principles}). From the \emph{Basic} MDP to \coolnameabbr{}, the state space increases exponentially with $\maxcommdelay$ from $|\mathrm{S}|$ to the order of $|\mathrm{S}|{(|\mathrm{Act}|+1)^{\maxcommdelay}}(\maxcommdelay+1)$. So, the time complexity of Algorithm \ref{alg:shield_design} for the \coolnameabbr{} also increases accordingly. However, \emph{the shield construction is an offline process}, and for practically observed delay values (Fig. \ref{fig:wifi_traces}), our approach scales well. We also provide a comprehensive analysis of the state space size, time taken to compute line 2 in Algorithm \ref{alg:shield_design}, and the size of the synthesized shield for the \coolnameabbr{} for our simulation environments in Table. \ref{tab:simulation_complexity}. Even for $\maxcommdelay = 3$, the time taken to compute the maximum safety probabilities is only close to an hour and the shield size is less than 50 MB. The effect of discretization on safety probability depends on the environment and the discretization method used, which is beyond the scope of this paper.

\section{Conclusion and future directions}
\label{sec:conclusion}

This paper provides a novel approach to accurately model the networked control system transitions, in the presence of stochastic communication delays, as an MDP. Consequently, we use the MDP to synthesize shields for safe networked control. We demonstrate the efficiency of our approach on simulation and hardware setups. Our work is timely since we are seeing a surge of teleoperated robots. As future work, we believe that exploring state space reduction techniques to handle the exponential growth of state space in \coolnameabbr{} and exploring solutions for continuous-time systems with delay using HJ reachability are promising directions.

\label{conclusion}

\section{Acknowledgements}

This work was supported in part by the Lockheed Martin Corporation and by the Office of Naval Research (ONR) under Grant No. N00014-22-1-2254. Additionally, this work received support from the National Science Foundation (NSF) under Grant No. 2148186 and is further supported by funding provided by Federal Agencies and Industry Partners as specified in the Resilient \& Intelligent NextG Systems (RINGS) program. The work solely reflects the opinions and conclusions of the authors and does not represent the views of any sponsor.

\bibliographystyle{IEEEtran}

\bibliography{ref/IEEEabrv, ref/external, ref/swarm}

\end{document}